\documentclass[letterpaper, oneside]{amsart}
\usepackage{geometry, color, graphicx,enumerate, framed,amsthm}

\def\RR{\ensuremath{\mathbb{R}}}
\newcommand{\CC}{{\mathbb C}}
\newcommand{\PP}{{\mathbb P}}

\newcommand{\rank}{\textup{rank}}

\newtheorem{theorem}{Theorem}
\numberwithin{theorem}{section} 
\numberwithin{equation}{section} 
\newtheorem{definition}[theorem]{Definition}
\newtheorem{lemma}[theorem]{Lemma}

\newtheorem{example}[theorem]{Example}
\newtheorem{corollary}[theorem]{Corollary}
\newtheorem{question}[theorem]{Question}

\newcommand{\FF}{{\mathbb{F}}}
\newcommand{\R}{{\mathbb{R}}}

\newcommand{\wt}{\widetilde}

\newcommand{\f}{\frac}

\newcommand{\lt}{\left}
\newcommand{\rt}{\right}

\newcommand{\degree}{\text{degree}}

\title{On the Existence of Epipolar Matrices}
\author{Sameer Agarwal}
\address{Google Inc. }
\email{sameeragarwal@google.com}

\author{Hon-Leung Lee}
\address{Department of Mathematics\\ University of Washington\\ Seattle, WA 98195}
\email{hllee@uw.edu}

\author{Bernd Sturmfels}
\address{Department of Mathematics\\  University of California\\ Berkeley, CA 94720}
\email{bernd@berkeley.edu}

\author{Rekha R. Thomas}
\address{Department of Mathematics\\ University of Washington\\ Seattle, WA 98195}
\email{rrthomas@uw.edu}

\thanks{Lee and Thomas were partially supported by NSF grant DMS-1418728, and Sturmfels by NSF grant DMS-1419018.} 

\begin{document}
\maketitle

\begin{abstract}
This paper considers the foundational question of the existence of a fundamental (resp. essential) 
matrix given $m$ point correspondences in two views. 
We present a complete answer for the existence of fundamental matrices for any value of $m$. Using examples we disprove the widely held beliefs that 
fundamental matrices always exist whenever $m \leq 7$. At the same time, we prove that they exist 
unconditionally when $m \leq 5$. Under a mild genericity condition, we show that an essential matrix always exists when $m \leq 4$. We also characterize the six and seven point configurations in two views for which all matrices satisfying the epipolar constraint have rank at most one.

\end{abstract}

\section{Introduction}
\label{sec:introduction}
A set of point correspondences $\{ (x_i, y_i) \in \RR^2 \times \RR^2, \,\,i=1,\ldots, m \}$
are the images of $m$ points in $\R^3$ in two uncalibrated (resp.~calibrated) cameras
   only if there exists a fundamental matrix $F$
   (resp.~essential matrix $E$)  such that the $(x_i,y_i)$
           satisfy the {\em epipolar constraints}~\cite[Chapter 9]{hartley-zisserman-2003}. Under mild genericity conditions on the point correspondences, the existence of these matrices is also sufficient for
the correspondences $(x_i,y_i)$ to be the images of a 3D scene
\cite{faugeras-92,hartley-uncalibrated-relative-92,hartley-uncalibrated-stereo-92,longuethiggins81}. This brings us to the following basic question in multiview geometry:
\begin{question}
\label{question1}
Given a set of $m$ point correspondences $(x_i,y_i) \in  \RR^2 \times \RR^2$,
when does there exist a fundamental (essential) matrix relating them via the epipolar constraints?
\end{question}

The answer to this question is known in several special cases~\cite{matas-et-al-2005,hartley-six-point}, but even in the 
minimally constrained and 
under-constrained cases ($m \le 7$ for fundamental matrices and $m\le5$ for essential matrices) our knowledge is incomplete.

For instance, in the uncalibrated case, for $m \le 7$, the popular statement of the so called {\em seven point algorithm} will have you believe that there always exists a fundamental matrix~\cite{hartley1994projective,sturm1869problem}. We will show that this is not true. The problem is, that the matrix returned by the seven point algorithm is only guaranteed to be rank deficient, it is not guaranteed to have rank two. 

In the calibrated case, when $m=5$, there exists up to 10 distinct complex essential matrices~\cite{demazure}, but it is not known when we can be sure that one of them is real. Similarly, 
it is unknown whether there always exists a real essential matrix for $m \le 4$ point correspondences. 

The common mistake in many of the usual existence arguments is the reliance on dimension counting. This unfortunately works only on algebraically closed fields, which the field of real numbers is not.

In this paper we give a complete answer to the existence question for fundamental matrices for all $m$ and for essential matrices for $m \le 4$. The problem of checking the existence of a fundamental (resp. essential) matrix for an arbitrary value of $m$ reduces to 
one where $m \leq 9$. The situations of $m=8,9$ are easy and thus the work needed is for $m \leq 7$. We prove the following results:
\begin{enumerate}[(1)]
\item For $m \le 5$ there always exists a fundamental matrix.\label{result1}  
\item For $m =6,7$ there may not exist a fundamental matrix, and we will provide an exact test for checking for its existence.\label{result2}
\item For $m \le 4$ there always exists an essential matrix. \label{result3}
\end{enumerate}

It is relatively easy to prove the existence of a fundamental matrix when $m \leq 4$. We give a much more sophisticated proof 
to extend this result to $m \leq 5$ in ~\eqref{result1}. Similarly, it is elementary to see that there is always an essential matrix when $m \le 3$. The proof of~\eqref{result3} is much more complicated.

A fundamental matrix can fail to exist in several ways. An important such case is when all matrices that run for competition have rank at most one. We fully characterize this phenomenon directly in 
terms of the geometry of the input point correspondences.

The key technical task in all this is to establish conditions for the existence 
of a real point in the intersection of a subspace and a fixed set of $3 \times 3$ matrices. 

In the remainder of this section we establish our notation and some basic facts about cameras, epipolar matrices, projective varieties and linear algebra. Section~\ref{sec:fundamental} considers the existence problem for the fundamental matrix and Section~\ref{sec:essential} does so for the essential matrix. We conclude in Section~\ref{sec:discussion} with a discussion of the results and directions for future work.

\subsection{Notation}

\label{sec:notation}
Capital roman letters (say $E, F, X, Y, Z$) denote matrices. For a matrix $F$, the corresponding lower case letter
$f$ denotes the vector obtained by concatenating the rows of $F$. Upper case calligraphic letters denote sets of matrices (say $\mathcal{E}, \mathcal{F}$).

For a field $\FF$ such as $\RR \textup{ or } \CC$, the projective space $\PP^n_\FF$ is $\FF^{n+1} \setminus \{0\}$ in which we identify $u$ and $v$ if $u = \lambda v$ for some $\lambda\in \FF\setminus\{0\}$. For example $(1,2,3)$ and $(4,8,12)$ are the same
point   in $\PP^2_\RR$, denoted as $(1,2,3) \sim (4,8,12)$.
The set of $m\times n$ matrices with entries in $\FF$ is denoted by $\FF^{m \times n}$, and by 
$\PP_\FF^{m\times n}$ if the matrices are only up to scale.
For $v \in \RR^3$,
$$
[v]_\times \,\,:= \,\,
\begin{pmatrix}
0 & -v_3 & v_2\\
v_3 & 0 & -v_1\\
-v_2 & v_1 & 0
\end{pmatrix}
$$
is a skew-symmetric matrix whose rank is two unless $v=0$. Also, $[v]_\times w = v \times w$, where $\times$ denotes the vector cross product.
For $A\in \FF^{m\times n}$, we have ${\rm ker}_\FF (A) = \{u \in \FF^n: A u = 0\}$,
and ${\rm rank}(A) =  n - \dim({\rm ker}_\FF (A))$. We use $\det(A)$ to denote the determinant of $A$.
Points $x_i$ and $ y_i$  in $\FF^2$ will be identified with their homogenizations
$(x_{i1},x_{i2},1)^\top$ and $(y_{i1},y_{i2},1)^\top $ in $ \PP_\FF^2$.
Also, 
$y_i^\top \otimes x_i^\top := \begin{pmatrix}
y_{i1}x_{i1} & y_{i1}x_{i2} & y_{i1} & y_{i2}x_{i1} & y_{i2}x_{i2} & y_{i2} & x_{i1} & x_{i2} & 1\end{pmatrix}\in \FF^{1\times 9}$.

If $P$ and $Q$ are finite dimensional subspaces, then $P \otimes Q$ is the span of the pairwise Kronecker products of the basis elements  of $P$ and $Q$.

\subsection{Linear algebra}

Below we list five facts from linear algebra that will be helpful in this paper.

\begin{lemma} \textup{\cite[pp. 399]{SempleKneebone}}\label{lem:non-collinear}
If $x_0,\ldots,x_{n+1}$
and $y_0,\ldots,y_{n+1}$ are two sets of $n+2$ points in 
 $\RR^{n+1}$ such that no $n+1$ points in either set are linearly dependent. 
 Then there is an invertible matrix $H\in \RR^{(n+1) \times (n+1)}$ such that 
$$
H x_i \sim y_i \ \text{ for any } i = 0,1,\ldots,n+1.
$$
\end{lemma}

\begin{lemma}{\rm\cite[Theorem 3]{meshulam}} \label{lem:meshulam}
Suppose $V$ is a linear subspace of $\RR^{n\times n}$ of dimension $rn$, such that 
for any $A\in V$, $\rank(A)\leq r$.
 Then either 
$V = W\otimes \RR^n$ or $V=\RR^n\otimes W$, for some $r$-dimensional subspace $W\subseteq \RR^n$.
\end{lemma}

\begin{lemma}{\rm\cite[Theorem 1]{flanders}} \label{lem:flanders}
Suppose $V$ is a linear subspace of $\RR^{m\times n}$ and $r$ is 
the maximum rank of an element of $V$. Then $\dim(V)\leq r \cdot \max \{m,n\}$.
\end{lemma}

\begin{lemma}[Matrix Determinant Lemma]{\rm\cite[Theorem 18.1.1]{harville}}\label{lem:matrix determinant} 
If $A\in \RR^{n\times n}$ is invertible and $u,v\in \RR^n$, then 
$
\det(A+uv^\top) = (1+v^\top A^{-1}u) \det(A).
$
\end{lemma}

In the following lemma we identify points in $\RR^2$ with their homogenizations in 
$\PP_\RR^2$ as mentioned earlier. The proof of the lemma is in Appendix~\ref{app:homography}.

\begin{lemma} \label{lem:homography}
Given two lines $l,m$ in $\RR^2$, and $x_0\in l$, $y_0\in m$,  there is an invertible matrix $H\in \RR^{3\times 3}$ such that 
\begin{enumerate}[(1)]
\item \label{lem:homography1}
$Hx_0 = y_0$; and 
\item \label{lem:homography2}
for any $x\in l$, $Hx\in m$.
\end{enumerate}
\end{lemma}

\subsection{Projective varieties}
\label{sec:projective}
We recall some basic notions from algebraic geometry~\cite{cox2007ideals,harris,shafarevich2013}. 
Let $\FF[u] = \FF[u_1,\ldots,u_n]$ denote the ring of all polynomials with coefficients in the field $\FF$.
\begin{definition}[Homogeneous Polynomial]
A polynomial in $\FF[u]$  is homogeneous (or called a {\em form}) if all its monomials have the same total degree.
\end{definition}
For example, $u_1^2 u_2 + u_1 u_2^2$ is a form of degree three but $u_1^3 + u_2$ is not a form.

\begin{definition}[Projective Variety and Subvariety]
A subset $\mathcal{V} \subseteq \PP_\FF^{n}$ is a projective variety if there are homogeneous polynomials $h_1,\ldots,h_t \in \FF[u]$ such that
$\mathcal{V} = \{u \in \PP^{n}_\FF: h_1(u) = \ldots = h_t(u) = 0\}$.  
A variety $\mathcal{V}_1$ is a subvariety of $\mathcal{V}$ if $\mathcal{V}_1 \subseteq \mathcal{V}$.
\end{definition}

Given homogeneous polynomials $h_1,\ldots,h_t \in \RR[u]$, let $\mathcal{V}_\CC := \{ u \in \PP_\CC^n \,:\, h_i(u)=0 \,\,\hbox{for} \,\,i=1,\ldots,t \}$ be their projective
variety over the complex numbers, and $\mathcal{V}_\RR := \mathcal{V}_\CC \cap \PP_\RR^n$ be the
set of real points in $\mathcal{V}_\CC$.

\begin{definition}[Irreducibility]
A projective variety $\mathcal{V}\subseteq \PP_\FF^n$  is irreducible if it is not the union of two nonempty proper subvarieties 
of $\PP_\FF^n$.
\end{definition}

We define the dimension of a projective variety over $\CC$ in a form that is particularly suitable to this paper. 

\begin{definition}[Dimension] {\rm\cite[Corollary 1.6]{shafarevich2013}}
The dimension $\dim (\mathcal{V})$ of a projective variety $\mathcal{V}\subseteq \PP_\CC^n$ is 
$d$ where $n-d-1$ is the maximum dimension of a linear subspace of $\PP_\CC^n$ disjoint from $\mathcal{V}$.
\end{definition}

As a special case, 
if $\mathcal{L}$ is a $l$-dimensional linear subspace in $\CC^{n+1}$, it can be viewed as an irreducible projective variety in $\PP_\CC^n$ 
of dimension $l-1$. 

The following result shows how dimension counting can be used to infer facts about the intersection of a variety and a linear subspace in  $\PP_\CC^n$. It is a consequence of the more general statement in \cite[Theorem 1.24]{shafarevich2013}. This result does not extend to varieties over $\RR$.

\begin{theorem} \label{thm:intersections}
Consider an irreducible projective variety $\mathcal{V}_\CC\subseteq \PP_n^\CC$ of dimension $d$  
and a linear subspace $\mathcal{L}\subseteq \PP_n^\CC$ of dimension $l$. If $d+l=n$ then $\mathcal{V}$ must intersect 
$\mathcal{L}$. If $d+l>n$ then $\mathcal{V}$ intersects $\mathcal{L}$ at infinitely many points.
\end{theorem}

Observe that the above theorem only applies over the complex numbers. As a simple illustration the curve $x^2 -y^2 +  z^2 = 0$ in $\PP_\CC^2$ is guaranteed to intersect the subspace $y=0$ in two complex points since they have complementary dimensions in $\PP_\CC^2$. 
 However, neither of these intersection points is real.
 
If $\mathcal{V}\subseteq \PP_\CC^n$ is a projective variety, then it intersects any linear subspace of dimension $n-\dim(\mathcal{V})$ in $\PP_\CC^n$. 
If the subspace is general, then the cardinality of this intersection is a constant which is an important invariant of the variety.
 
\begin{definition}[Degree]{\rm\cite[Definition 18.1]{harris}} \label{def:degree}
The  degree of a  projective variety $\mathcal{V} \subseteq \PP^{n}_\CC$, denoted by $\degree (\mathcal{V})$, is the number of intersection points with a general linear subspace of dimension $n- \dim(\mathcal{V})$  in $\PP_\CC^n$ .
\end{definition}

\subsection{Camera Matrices}
A general projective camera can be modeled by a matrix $P \in \PP_\RR^{3 \times 4}$ with $\rank(P)=3$. Partitioning a camera as $P=\begin{pmatrix} A & b \end{pmatrix}$ where $A \in
\RR^{3 \times 3}$, we say that $P$ is a {\em finite camera} if $A$ is
nonsingular. 
In this paper we restrict ourselves to finite cameras.

A finite camera $P$ can be written as $P = K\begin{pmatrix} R & t\end{pmatrix}$,
where $t \in \RR^3$,
 $K$ is an upper triangular matrix with positive diagonal entries known as the {\em calibration matrix}, and $R \in \textup{SO}(3)$ is a rotation matrix that represents the orientation of the camera coordinate frame.
    If the calibration matrix $K$ is known, then the camera is
    said to be {\em calibrated}, and otherwise the camera is {\em uncalibrated}. The {\em normalization} of a calibrated camera $P = K\begin{pmatrix} R & t\end{pmatrix}$ is the camera $K^{-1}P = \begin{pmatrix} R & t\end{pmatrix}$. 
    
By dehomogenizing
(i.e.~scaling the last coordinate to be $1$),
 we can view the image $x = Pw$ as a point in~$\RR^2$. If $x$ is the image of $w$ in the calibrated camera $P$, then $K^{-1}x$ is called the {\em normalized image} of $w$, or equivalently, it is the image of $w$ in the normalized camera $K^{-1}P$. This allows us to remove the effect of the calibration $K$ by passing to the normalized camera $K^{-1}P$ and  normalized images $\tilde{x} := K^{-1}x$.

\subsection{Epipolar Matrices}
\label{sec:epipolar}

 In this paper we use the name {\em epipolar matrix} to refer to either a {fundamental matrix} or {essential matrix}
derived from the {\em epipolar geometry} of a pair of cameras. These matrices are explained and studied in \cite[Chapter 9]{hartley-zisserman-2003}.

An {\em essential matrix} is any matrix in $\PP_\RR^{3 \times 3}$ of the form $E = SR$ where $S$ is a skew-symmetric matrix and $R \in \textup{SO}(3)$. Essential matrices are characterized by the property that they have rank two (and hence one zero singular value) and two equal non-zero singular values. An essential matrix depends on six parameters,
three each from $S$ and $R$, but since it is only defined up to scale, it has five degrees of freedom.

The essential matrix of the two normalized cameras $\begin{pmatrix} I & 0 \end{pmatrix}$ and 
$\begin{pmatrix} R & t \end{pmatrix}$ is $E = [t]_\times R$.
For every pair of normalized images $\tilde{x}$ and $\tilde{y}$ in these cameras of a point 
$w \in \PP_\RR^3$, the triple $(\tilde{x}, \tilde{y}, E)$ satisfies the {\em epipolar constraint}
\begin{align} \label{eq:epipolar for E}
\tilde{y}^\top E \tilde{x} \,\,=\,\, 0.
\end{align}
Further, any $E=SR$ is the essential matrix of a pair of cameras as shown in \cite[Section 9.6.2]{hartley-zisserman-2003}.

If the calibrations $K_1$ and $K_2$
of the two cameras were unknown, then for a pair of corresponding images $(x,y)$ in the two cameras, the epipolar constraint becomes
\begin{align} \label{eq:epipolar for F}
0 \,=\, \tilde{y}^\top E \tilde{x} \,=\, y^\top K_2^{-\top} E K_1^{-1} x.
\end{align}
The matrix $F := K_2^{-\top} E K_1^{-1}$
 is the {\em fundamental matrix} of the two uncalibrated cameras. This is a rank two matrix but its two non-zero singular values are no longer equal.
 Conversely, any real $3 \times 3$ matrix of rank two is the fundamental matrix of a pair of cameras \cite[Section 9.2]{hartley-zisserman-2003}. A fundamental matrix has seven degrees of freedom since it satisfies the rank two condition and is only defined up to scale.
The set of fundamental matrices can be parametrized as $F = [b]_\times H$, where $b$ is a non-zero vector and $H$ is an 
invertible matrix $3 \times 3$ matrix \cite[Section 9.6.2]{hartley-zisserman-2003}.

\subsection{$X,Y$ and $Z$}
\label{sec:xyz}
Suppose we are given  $m$ point correspondences (normalized or not) $\{ (x_i, y_i), i=1,\ldots,m\}
\subseteq \RR^2 \times \RR^2$. We homogenize this data and represent it by
  three matrices with $m$ rows:
\begin{align} 
X &= \begin{pmatrix}
x_1^\top\\
\vdots\\
x_m^\top
\end{pmatrix} \in \RR^{m \times 3},\\
Y &= \begin{pmatrix}
y_1^\top\\
\vdots\\
y_m^\top
\end{pmatrix} \in \RR^{m \times 3},\ \textup{and }\\
Z &= \begin{pmatrix}
y_1^\top \otimes x_1^\top\\
\vdots\\
y_m^\top \otimes x_m^\top
\end{pmatrix} \in \RR^{m \times 9}.  \label{formula of Z}
\end{align}

The ranks of $X$ and $Y$ are related to the geometry of the point sets $\{x_i\}$ and $\{y_i\}$. This is made precise by the following lemma which is stated in terms of $X$ but obviously also applies to $Y$.
\begin{lemma} \label{lem:geometry of rankX}
\[
\rank(X)=\begin{cases}
1 & \mbox{ If $x_i$'s, as points in $\RR^2$, are all equal.}\\
2 & \mbox{ If all the $x_i$'s are collinear in $\RR^2$ but not all equal.}\\
3 & \mbox{ If the $x_i$'s are noncollinear in $\RR^2$.}
\end{cases}
\]
\end{lemma}

Notice that every row of $X$ (resp. $Y$) can be written as a linear combination of $\rank(X)$ (resp. $\rank(Y)$) rows of it.
Using this and the bilinearity of Kronecker product, it is evident that:
\begin{lemma} \label{lem:rank bound}
For any $m$, 
$$\rank(Z)\leq \rank(X)\,\rank(Y) \leq 9.$$
\end{lemma}
In particular, if all points $x_i$ are collinear in $\RR^2$ then $\rank(Z)\leq 6$. If all points $x_i$ are equal in $\RR^2$ then $\rank(Z)\leq 3$.

We study Question~\ref{question1} via the the subspace $\ker_\RR(Z)$. Observe that for all $m$,  $\ker_\RR(Z) = \ker_\RR(Z')$ for a supmatrix $Z'$ of $Z$ consisting of $\rank(Z)$ linearly independent rows. Therefore, we can replace $Z$ with $Z'$ in order to study $\ker_\RR(Z)$ which allows us to restrict our investigations to the values of $m$ such that 
\begin{framed}
\begin{equation}
1 \le m = \rank(Z) \le 9. \label{eqn:rank Z is m}
\end{equation}
\end{framed}

In light of the above discussion, it is useful to keep in mind that even though all our results are stated in terms of $m \le 9$, we are in fact covering all values of $m$.

\section{Fundamental Matrices}
\label{sec:fundamental}

Following Section \ref{sec:epipolar},
a fundamental matrix is any matrix in $\PP_\RR^{3 \times 3}$ of rank two~\cite[Section 9.2.4]{hartley-zisserman-2003}.  In our notation, we denote  the set of
fundamental matrices as 
\begin{equation}
\mathcal{F} \,\,:= \,\,\{ f \in \PP_\RR^8 \,:\, \rank(F) = 2 \},
\end{equation}
where the vector $f$ is
the concatenation of the rows of the matrix $F$. This notation allows us to write the epipolar constraints (\ref{eq:epipolar for F})
as 
\begin{equation}
Zf = 0.
\end{equation}

Hence a fundamental matrix $F$ exists for the
$m$ given point correspondences
  if and only if the linear subspace
 $\ker_\RR(Z)$ intersects the set $\mathcal{F}$, i.e., 
\begin{align} 
\ker_\RR(Z) \cap \mathcal{F} \neq \emptyset.
\end{align}
This geometric reformulation of the existence question for $F$ is well-known in multiview geometry 
\cite{hartley-zisserman-2003,MaybankBook}.

We now introduce two 
complex varieties that are closely related to $\mathcal{F}$. 

Let $\mathcal{R}_1 := \{a \in \PP_\CC^8 \,:\, \rank(A) \leq 1 \}$ be the set of matrices in $\PP_\CC^{3 \times 3}$ of rank one. 
It is an irreducible variety with $\dim(\mathcal{R}_1) = 4$ and $\degree(\mathcal{R}_1)=6$. 

Let $\mathcal{R}_2 := \{ a\in \PP_\CC^8 \,:\, \rank(A) \leq 2 \}$ be the set of matrices in $\PP_\CC^{3 \times 3}$ of rank at most two. It is an irreducible variety  with $\dim(\mathcal{R}_2) = 7$ and $\degree(\mathcal{R}_2)=3$. Observe that
\begin{equation}
\mathcal{R}_2 = \{ a \in \PP_\CC^8 \,:\, \det(A) = 0 \}.
\end{equation}

The set of fundamental matrices can now be written as $\mathcal{F} = (\mathcal{R}_2 \backslash \mathcal{R}_1) \cap \PP_\RR^8$ which is not a variety over  $\RR$.

In this section we will give a complete answer to the question of existence of fundamental matrices for any number $m$ of point correspondences.  Recall from Section~\ref{sec:xyz} \eqref{eqn:rank Z is m} that assuming $m = \rank(Z)$, we only need to consider the cases $1\le m \le 9$.

\subsection{Case: $m=9$}
If $m = 9$, then $\ker_\RR(Z)\subseteq \PP_\RR^8$ is empty, and $Z$ has no fundamental matrix.

\subsection{Case: $m=8$}
If $m=8$, then $\ker_\RR(Z)\subseteq \PP_\RR^8$ is a point $a \in \PP_\RR^8$ corresponding to the matrix $A \in \PP_\RR^{3 \times 3}$. It is possible for  
$A$ to have rank one, two or three. Clearly, $Z$ has a fundamental matrix if and only if $A$ has rank two.

\subsection{Case: $m=7$}
\label{sec:F m = 7}
The majority of the literature in computer vision deals with the case of $m= 7$ which falls under the category of  ``minimal problems''; 
see for example \cite[Chapter 3]{stewenius_thesis}. 
The name refers to the fact that $m=7$ is the smallest value of $m$ for which $\ker_\CC(Z) \cap \mathcal{R}_2$ is finite, making the problem of estimating $F$ well-posed (at least over $\CC$). 

Indeed, when $m=7$, $\ker_\CC(Z)$ is a one-dimensional subspace of $\PP^8_\CC$ and hence by Theorem~\ref{thm:intersections}, generically it will intersect $\mathcal{R}_2$ in three points, of which at least one is real since $\det(A)$ is a degree three polynomial.  
Therefore, there is always a matrix of rank at most two in $\ker_\RR(Z)$. This leads to the common belief that when $m=7$, there is always a fundamental matrix for $Z$. 

We first show an example for which $\ker_\RR(Z)$ contains only matrices of ranks either one or three.
\begin{example}
\label{example:rank one and three}
{\rm
Consider 
$$
X = \left( \begin{array}{ccc}
\frac{1}{5} & -1 & 1 \\
-1 & -7 & 1\\
\frac{-1}{2} & 0 & 1\\
-2 & -12 & 1 \\
\frac{-57}{4} & 8 & 1 \\
2 & 8 & 1 \\
0 & \frac{-1}{9} & 1
\end{array}
\right)
\textup{ and }
Y = \left( \begin{array}{ccc}
0 & 1 & 1 \\
1 & 0 & 1\\
 2 & 5 & 1  \\
 3 & \frac{-5}{12} & 1 \\
 4 & 7 & 1 \\
 5 & \frac{-11}{8} & 1 \\
 6 & 9 & 1
 \end{array}
\right).
$$
Here, $\rank(Z)=7$ and $\ker_\RR(Z)$ is spanned by the rank three matrices 
$$
I = \begin{pmatrix}
1 & 0 & 0 \\
0 & 1 & 0 \\
0 & 0 & 1 
\end{pmatrix} \text{ and } A_2=
\begin{pmatrix}
0 & 1 & 2 \\
5 & 4 & -2 \\
-15 & 3 & 11 
\end{pmatrix}.
$$
For any $u_1,u_2\in \RR$,  one obtains 
$$\det(Iu_1 + A_2u_2) = (u_1 + 5u_2)^3.$$ 
If $\det(Iu_1 + A_2u_2) =0$, then $u_1=-5u_2$ and 
$$
Iu_1 + A_2u_2 = u_2(A_2-5I) = u_2 
\left(\begin{array}{ccc}
-5 & 1 & 2 \\ 5 & -1 & -2 \\ -15 & 3 & 6 \end{array} \right)
$$
which has rank at most one.
Thus for $(u_1,u_2)\neq (0,0)$, 
$$
\rank(Iu_1+A_2u_2) = \begin{cases}
1  & \mbox{ if } u_1+5u_2=0\\
3  & \mbox{ if } u_1+5u_2\neq 0.
\end{cases}
$$
Hence ${\rm ker}_\RR(Z)$ consists of matrices of rank either one or three, and $Z$ does not have a fundamental matrix.
}
\end{example}

Another way for $Z$ to not have a fundamental matrix is if $\ker_\RR(Z)$ is entirely in $\mathcal{R}_1$.
The following theorem whose proof can be found in Appendix~\ref{sec:m=7 and kernel in rank one}, characterizes this situation.
See Figure~\ref{fig:m=7 and kernel in rank one} for illustrations.

\begin{figure*}[t!]
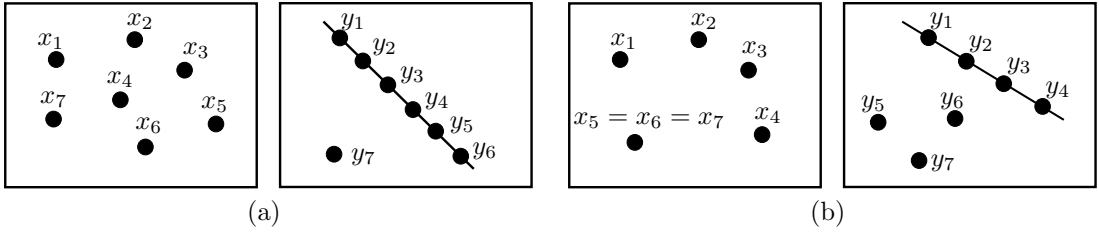

\centering
\begin{tabular}{cc}
\def\svgwidth{200pt} \input{example1_pdf.tex} &
\def\svgwidth{200pt} \input{example2_pdf.tex}\\
(a) & (b)
\end{tabular}
\caption{Two examples where the conditions for Theorem~\ref{thm:m=7 and kernel in rank one} are satisfied and there does not exist a fundamental matrix for $m = 7$ because $\ker_\RR(Z)\subseteq \mathcal{R}_1$.}
\label{fig:m=7 and kernel in rank one}
\end{figure*}

\begin{theorem}  \label{thm:m=7 and kernel in rank one}
If $m=7$, $\ker_\RR(Z) \subseteq \mathcal{R}_1$ if and only if one of the following holds:
\begin{enumerate}[(1)]
\item There is a nonempty proper subset $\tau \subset \{1,\ldots,7\}$ such that as points in $\RR^2$, $\{y_i \,:\, i \in \tau \}$ are collinear and 
$x_i = x_j$ for all $i,j \notin \tau$. \label{thm:m=7 and kernel in rank one, part 1}
\item There is a nonempty proper subset $\tau \subset \{1,\ldots,7\}$ such that as points in $\RR^2$, $\{x_i \,:\, i \in \tau \}$ are collinear and 
$y_i = y_j$ for all $i,j \notin \tau$.\label{thm:m=7 and kernel in rank one, part 2}
\end{enumerate}
\end{theorem} 

\subsection{Case: $m = 6$}
\label{sec:F m = 6}
When $m \leq 6$, by Theorem~\ref{thm:intersections}, $\ker_\CC(Z) \cap \mathcal{R}_2$ is infinite, and here the conventional wisdom is that there are infinitely many fundamental matrices for $Z$ and thus these cases deserve no study. 

Indeed, it is true that for six points in two views in general position, there exists a fundamental matrix relating them. To prove this, we first note the following fact which is a generalization of a result of Chum et al.~\cite{matas-et-al-2005}. 
Its proof can be found in Appendix~\ref{sec:chum}.

\begin{lemma}\label{lemma:chum}
Assume there is a real $3\times 3$ invertible matrix $H$ such that 
for at least $m-2$ of the indices $i \in \{1,\ldots,m\}$, $y_i \sim H x_i$. Then $Z$ has a fundamental matrix.
\end{lemma}

An immediate consequence of Lemma~\ref{lemma:chum} with $m=6$ and 
Lemma~\ref{lem:non-collinear} with $n=2$ is the following.
 
\begin{theorem} \label{thm:F generically for m=6}
If $m=6$ and $\tau$ is a subset of $1,\ldots,6$ with four elements such that no set of three points in either   
$\{x_i \ : \  i\in \tau\}$ or $\{y_i \ : \  i \in \tau\}$ is collinear in $\RR^2$, then $Z$ has a fundamental matrix. 
\end{theorem}

Note that in \cite[Theorem 1.2]{hartley-six-point}, Hartley shows that a fundamental matrix associated with 
six point correspondences is uniquely determined under certain geometric assumptions on 
the point correspondences and world points. One of Hartley's assumptions is the assumption of Theorem~\ref{thm:F generically for m=6}.

Theorem~\ref{thm:F generically for m=6} hints at the possibility that collinearity of points in any one of the two views may prevent a fundamental matrix from existing. The following theorem, whose proof can be found in Appendix~\ref{sec:m=6 and kernal in rank one}, characterizes the conditions under which ${\rm ker}_\RR(Z) \subseteq \mathcal{R}_1$ when $m=6$. No fundamental matrix can exist in this case.

\begin{theorem} \label{thm:m=6 and kernal in rank one}
If $m=6$, ${\rm ker}_\RR(Z) \subseteq \mathcal{R}_1$ if and only if either
all points $x_i$ are collinear in $\RR^2$ or all points $y_i$ are collinear in $\RR^2$.
\end{theorem}

We remark that when $m=6$, it is impossible that as points in $\RR^2$, all $x_i$ are collinear and all $y_i$ are collinear. 
If this were the case, then by Lemmas~\ref{lem:geometry of rankX} and \ref{lem:rank bound}, $\rank(Z)\leq 4<6=m$ which violates our assumption 
\eqref{eqn:rank Z is m}.

\subsection{Existence of fundamental matrices in general}

In the previous two sections, we have demonstrated that dimension counting is not enough
to argue for the existence of a fundamental matrix for $m = 6$ and $m = 7$. We have also described particular configurations in two views which guarantee the existence and non-existence of a fundamental matrix. We are now ready to tackle the general existence question for fundamental matrices for $m \le 8$. To do this,  we first need the following key structural lemma. It provides a sufficient condition for $\ker_\RR(Z)$ to have a matrix of rank two.

\begin{lemma} \label{lem:bernd}
Let $\mathcal{L}$ be a positive dimensional subspace in $\PP_\RR^{3 \times 3}$ that contains a matrix of rank three. If  the determinant  restricted to $\mathcal{L}$ is not a power of a linear form, then $\mathcal{L}$ contains a real matrix of rank two.
\end{lemma}

The proof of this lemma can be found in Appendix~\ref{app:lemma bernd}, but we elaborate on its statement.
 If $\{A_1, \ldots, A_t\}$ is a basis of a subspace $\mathcal{L}$ in $\PP_\RR^{3 \times 3}$, then any matrix in $\mathcal{L}$ is of the form $A = u_1 A_1 + \cdots + u_t A_t$ for scalars $u_1, \ldots, u_t \in \RR$,  and $\det(A)$ is a polynomial in $u_1, \ldots, u_t$ of degree at most three. Lemma~\ref{lem:bernd} says that if $\det(A)$ is not a power of a linear form $a_1 u_1 + \cdots + a_t u_t$ where $a_1, \ldots, a_t \in \RR$, then $\mathcal{L}$ contains a matrix of rank two.  

It is worth noting that~Lemma~\ref{lem:bernd} is only a sufficient condition and not necessary for a subspace $\mathcal{L} \subseteq \PP_\RR^{3 \times 3}$ 
 to have a rank two matrix. This is illustrated by the following example:
 
\begin{example}
For 
$$
X = \begin{pmatrix}
-1 & 0 & 1 \\
-3 & 0 & 1\\
6 & 3 & 1\\
0 & 1 & 1 \\
2 & 2 & 1 \\
0 & \f{1}{2}& 1 \\
\f{1}{2} & 1 & 1
\end{pmatrix}
\textup{ and }
Y = \begin{pmatrix}
1 & 0 & 1 \\
\f{1}{3} & 0 & \,1\\
 \f{1}{3} & -1 & \,1  \\
1 & -1 & \,1 \\
 \f{1}{2} & -1 & \,1 \\
 4 & -2 & \,1 \\
 2 & -2 & \,1
 \end{pmatrix}, 
$$
${\rm ker}_\RR (Z)$ is spanned by 
$$ A_1 =
 \begin{pmatrix} 1 & 0 & 0 \\ 0 & 1 & 0 \\ 0 & 0 & 1\end{pmatrix} \textup{ and } 
 A_2 = \begin{pmatrix} 0 & 1 & 0 \\ 0 & 0 & 1 \\ 0 & 0 & 0\end{pmatrix},$$ and
$\det(A_1 u_1 +A_2 u_2) = u_1^3$. Since $\rank(A_2)=2$, $A_2$ is a fundamental matrix of $Z$.
 \end{example}

We now present a general theorem that characterizes the existence of a fundamental matrix 
for  $m\leq 8$.

\begin{theorem} \label{thm:m=6,7}
For a basis $\{A_1, \ldots, A_t\}$ of $\ker_\RR(Z)$, define $M(u) := \sum_{i=1}^t A_i u_i$, and set 
$d(u) := \det(M(u))$. 

\begin{enumerate}
\item If $d(u)$ is the zero polynomial  then  
$Z$ has a fundamental matrix if and only if 
some $2\times 2$ minor of $M(u)$ is nonzero.

\item If $d(u)$ is a nonzero polynomial that is not a power of a linear form in $u$ 
then $Z$ has a fundamental matrix. 

\item If $d(u) = (b^\top u)^k$ for some $k\geq 1$ and non-zero vector $b \in \RR^t$, 
then $Z$ has a fundamental 
matrix if and only if some $2\times 2$ minor of $M\bigl(u - \frac{b^\top u}{b^\top b} b\bigr)$ is nonzero.
\end{enumerate}

\end{theorem}

\begin{proof}
Note that $M(u)$ is a parametrization of $\ker_\RR(Z)$ and  $d(u)$ is a polynomial in $u$ of degree at most three. 

\begin{enumerate}
\item If $d(u)$ is the zero polynomial, then 
all matrices in $\ker_\RR(Z)$ have rank at most two.
In this case, $Z$ has a fundamental matrix if and only if 
some $2\times 2$ minor of $M(u)$ is a nonzero polynomial in $u$.

\item If $d(u)$ is a nonzero polynomial in $u$, then we factor $d(u)$ and see if it is the cube of a linear form. If it is not, then by Lemma~\ref{lem:bernd}, $Z$ has a fundamental matrix.

\item Suppose  $d(u) = (b^\top u)^k$ for some $k\geq 1$ and non-zero vector $b$. 
Then the set of rank deficient matrices in ${\rm ker}_\RR(Z)$ is
$\mathcal{M} := \left\{ M(u) : u \in b^\perp \right\}$
where, $b^\perp := \{u \in \RR^t\ : b^\top u = 0\}$.
The hyperplane $b^\perp$ consists of all vectors $\,u - \frac{b^\top u}{b^\top b} b \,$ where $ u \in \RR^t$.
Therefore, $\mathcal{M} = \left\{ M\bigl(u - \frac{b^\top u}{b^\top b} b \bigr) : u \in \RR^{t} \right\}$.
As a result, 
$Z$ has a fundamental matrix if and only if some $2\times 2$ minor of $M\bigl(u - \frac{b^\top u}{b^\top b} b\bigr)$ is nonzero.  
\end{enumerate}
\end{proof}

\subsection{Cases: $m \leq 5$}
 
While Theorem~\ref{thm:m=6,7} provides a general existence condition for fundamental matrices for $m \le 8$, 
we now show that for $m\leq 5$ there always exists a fundamental matrix.

\begin{theorem} \label{thm:F exists for m<=5}
Every three-dimensional subspace of $\PP_\RR^{3\times 3}$ contains a rank two matrix. In particular, if $m \leq 5$, then $Z$ has a fundamental matrix.
\end{theorem}

\begin{proof} 
Suppose $\mathcal{L}$ is a three-dimensional subspace in $\PP_\RR^{3 \times 3}$ generated by the basis $\{A_1,\ldots,A_4\}$, and suppose $\mathcal{L}$ does not contain a rank two matrix.  Since the dimension of $\mathcal{L}$ as a linear subspace is four, by 
applying Lemma~\ref{lem:flanders} with $m=n=3$ and $r=1$, we see that $\mathcal{L}$ cannot be contained in the variety of rank one matrices. Therefore, we may assume that 
$A_4$ has rank three. Since $\mathcal{L}$ is assumed to have no matrices of rank two, by Lemma~\ref{lem:bernd} we also have that 
\begin{align} \label{eq:cubic}
\det(\lambda_1 A_1 + \lambda_2 A_2 +\lambda_3 A_3 +\lambda_4 A_4) 
=  (a_1 \lambda_1 +a_2\lambda_2+a_3\lambda_3 +a_4\lambda_4)^3
\end{align}
where $\lambda_1,\cdots,\lambda_4$ are variables. Note that $a_4 \neq 0$ since otherwise, choosing 
$\lambda_1 = \lambda_2 = \lambda_3 = 0$ and $\lambda_4 = 1$ we get $\det(A_4) = 0$ which is impossible.

By a change of coordinates, we may assume that 
\begin{align} \label{eq:change of coords}
\text{$\det(\lambda_1 A_1 + \lambda_2 A_2 +\lambda_3 A_3 +\lambda_4 A_4) = \lambda_4^3$,}
\end{align}
and in  particular, $\det(A_4)=1$. Indeed, 
consider
\begin{align}
\wt{A}_1  := A_1 - \f{a_1}{a_4} A_4, \ \wt{A}_2  := A_2  -\f{a_2}{a_4} A_4, \
\wt{A}_3  := A_3 - \f{a_3}{a_4} A_4, \ \wt{A}_4  := \f{1}{a_4} A_4
\end{align}
which also form a basis of  $\mathcal{L}$. Then using \eqref{eq:cubic} with the variables $\eta_1,\eta_2,\eta_3,\eta_4$, we obtain
\begin{align*}
 & \ \det(\eta_1 \wt{A}_1+ \eta_2 \wt{A}_2 + \eta_3 \wt{A}_3 +\eta_4 \wt{A}_4) \\
=&  \ \det \lt( \eta_1 A_1 + \eta_2 A_2 + \eta_3 A_3 + \f{\lt( \eta_4 - \eta_1a_1- \eta_2 a_2- \eta_3 a_3 \rt)}{a_4} A_4 \rt) \\
= & \ \lt( a_1 \eta_1 + a_2 \eta_2 + a_3 \eta_3 + \lt( \eta_4 - \eta_1a_1- \eta_2 a_2- \eta_3 a_3 \rt) \rt)^3 = \eta_4^3, 
\end{align*}
which is the desired conclusion.

Setting $\lambda_4 = 0$ in \eqref{eq:change of coords} we get 
$\det(\lambda_1A_1 + \lambda_2A_2+\lambda_3A_3) = 0 $. Hence,
${\rm span}\{A_1,A_2,A_3\}$ consists only of rank one matrices since there are no rank two matrices in $\mathcal{L}$.
Therefore, by Lemma~\ref{lem:meshulam} with $n=3$ and $r=1$, up to taking transposes of all $A_i$, 
there are column vectors $u,v_1,v_2,v_3\in \PP_\RR^2$ such that $A_j = uv_j^\top$ for all $j=1,2,3$.

Now setting $\lambda_4=1$,  by the Matrix Determinant Lemma, we have 
\begin{align*}
1& = \det(\lambda_1 A_1 + \lambda_2 A_2 +\lambda_3 A_3 + A_4)\\
& = \det(A_4 + u(\lambda_1 v_1^\top + \lambda_2 v_2^\top + \lambda_3 v_3^\top))\\
& = 1+(\lambda_1 v_1^\top + \lambda_2 v_2^\top + \lambda_3 v_3^\top)A_4^{-1}u.
\end{align*}
Hence $(\lambda_1 v_1^\top + \lambda_2 v_2^\top + \lambda_3 v_3^\top)A_4^{-1}u $ is the zero polynomial, and so $A_4^{-1}u$ is a non-zero vector orthogonal to ${\rm span}\{v_1,v_2,v_3\}$. This means that 
$v_1,v_2,v_3$ are linearly dependent, and so are $A_1,A_2,A_3$, which is impossible.
This completes the proof of the first statement.

If $m\leq 5$, then $\rank(Z)\leq 5$ (cf.~\eqref{eqn:rank Z is m}) and so $\ker_\RR(Z)$ is a subspace in $\PP_\RR^8$ of dimension at least three. 
By the first statement of the theorem, $Z$ has a fundamental matrix. 
\end{proof}

Note that 
when $m \leq 4$ there is a simpler proof (Appendix~\ref{appendix:fundamental}) for the existence of a fundamental matrix associated to the point correspondences, but it does not extend to the case of $m=5$.

\subsection{Comments}
As far as we know, the seven and eight point algorithms are the only general methods for checking the existence of a fundamental matrix. They work by first computing the matrices in $\ker_\CC(Z) \cap \mathcal{R}_2$ and then 
checking if there is a real matrix of rank two in this collection.  

While such an approach might decide the existence of a fundamental matrix for a given input $X$ and $Y$, it does not shed light on the 
structural requirements of $X$ and $Y$ to have a fundamental matrix. The goal of this paper is to understand the existence of epipolar matrices in terms of the input data.





When the input points $x_i$ and $y_i$ are rational, the results in this section also certify the existence of a fundamental matrix by exact rational arithmetic in polynomial time. The only calculation that scales with $m$ is the computation of a basis for $\ker_\RR(Z)$, which can be done in polynomial time using Gaussian elimination.

\section{Essential Matrices}
\label{sec:essential}
We now turn our attention to calibrated cameras and essential matrices. 
The set of essential matrices is the set of real $3 \times 3$ matrices of rank two with two equal (non-zero) singular values~\cite{maybank-faugeras}. 
In particular, all essential matrices are fundamental matrices and hence, 
contained in $\mathcal{R}_2 \setminus \mathcal{R}_1$.
We denote the set of 
essential matrices by
\begin{equation}
\label{eq:singvalues}
	\mathcal{E}_\RR = \{e \in \PP^8_\RR : \sigma_1(E) = \sigma_2(E)
	\,\,\,{\rm and} \,\,\, \sigma_3(E) = 0\},
\end{equation}
where $\sigma_i(E)$ denotes the $i^{\rm{th}}$ singular value of the matrix $E$. 
Demazure~\cite{demazure} showed that 
\begin{align} \label{E_describe}
\mathcal{E}_\RR \,= \bigl\{e \in \PP^8_\RR : p_j(e) = 0 \,\,\, {\rm for} \,\,\, j = 1, \hdots, 10 \bigr\},
\end{align}
where the $p_j$'s are homogeneous polynomials of degree three defined as
\begin{align}
\begin{pmatrix}
p_1 & p_2 & p_3 \\
p_4 & p_5 & p_6\\
p_7 & p_8 & p_9
\end{pmatrix} &:=  2 EE^\top E - {\rm Tr}(EE^\top) E,  
\textup{ and }\\
p_{10} &:= \det(E). 
\end{align}
Therefore, $\mathcal{E}_\RR$ is a real projective variety in $\PP_\RR^8$.

Passing to the common complex roots of the cubics $p_1, \ldots, p_{10}$, we get
\begin{align}
\mathcal{E}_\CC := \{e \in \PP^8_\CC : p_j(e) = 0, \, \forall j = 1, \hdots, 10\}.
\end{align}
This is an irreducible projective variety
with $\dim(\mathcal{E}_\CC) = 5$ and $\degree(\mathcal{E}_\CC) = 10$ (see~\cite{demazure}), and $\mathcal{E}_\RR = \mathcal{E}_\CC \cap \PP_\RR^8$.
See \cite{MaybankBook} for many interesting facts about $\mathcal{E}_\CC$ and $\mathcal{E}_\RR$ and their role in reconstruction problems in multiview geometry.

As before, our data consists of  $m$ point correspondences, which are now normalized image coordinates. For simplicity we will denote them as $\{ (x_i, y_i), \,\,i=1,\ldots,m\}$ instead of $\{ (\tilde{x}_i, \tilde{y}_i), \,\,i=1,\ldots,m\}$.

As in the uncalibrated case,  we can write the epipolar constraints  
(cf. \eqref{eq:epipolar for E}) as $Ze = 0$ where $e\in \mathcal{E}_\RR$, and 
$Z$ has an essential matrix if and only if
\begin{align} 
\ker_\RR(Z) \,\cap\, \mathcal{E}_\RR\,\, \neq \,\, \emptyset.
\end{align}
Hence the existence of an essential matrix for a given $Z$
is equivalent to the intersection of a subspace with a fixed real projective variety being non-empty. This formulation can also be found in \cite[Section 5.2]{MaybankBook}.

\subsection{Cases: $m=8,9$}
As in the previous section it is easy to settle the existence of $E$ for $m=8,9$.
If $m=8$, then the subspace $\ker_\RR(Z)\subseteq \PP_\RR^8$ is a point $a$
in $\PP_\RR^8$, and $Z$ has an essential 
matrix if and only if $A$ satisfies the conditions of \eqref{eq:singvalues} or \eqref{E_describe}.  If $m = 9$, then $\ker_\RR(Z)\subseteq \PP_\RR^8$ is empty, and $Z$ has no essential matrix.

\subsection{Cases: $5 \leq m \leq 7$}
The ``minimal problem'' for essential matrices is the case of $m=5$ where, by Definition~\ref{def:degree}, $\mathcal{E}_\CC \cap \ker_\CC(Z)$ is a finite set of points. Since $\degree(\mathcal{E}_\CC) = 10$, generically we expect ten distinct complex points in this intersection. An essential matrix exists for $Z$ if and only if one of these points is real. There can be ten distinct real points in $\mathcal{E}_\CC \cap \ker_\CC(Z)$ as shown in \cite[Theorem 5.14]{MaybankBook}. On the other extreme,  it can also be that no point in $\mathcal{E}_\CC \cap \ker_\CC(Z)$  is real as we show below. 

\begin{example}\label{ex:all complex Es}
We verified using \verb+Maple+ that the following set of five point correspondences has no essential matrix. $$ 
X = \left(\begin{array}{rrr}
3 & 0 & 1 \\
9 & 1 & 1 \\
1 & 2 & 1 \\
8 & 8 & 1 \\
4 & 8 & 1
\end{array}\right),\quad
Y = \left (\begin{array}{rrr}
2 & 0 & 1 \\
5 & 4 & 1 \\
 9 & 6 & 1 \\
 2 & 5 & 1 \\
 1 & 4 & 1 \\
\end{array}\right).
$$ 
None of the ten points in
$\ker_\CC(Z) \cap \mathcal{E}_\CC$ are real. 
\end{example}

As we have mentioned earlier, the existence of an essential matrix is equivalent to existence of a real point in the intersection $\ker_\CC(Z) \cap \mathcal{E}_\CC$. In general, this is a hard question which falls under the umbrella of {\em real algebraic geometry}.

The reason we were able to give a general existence result for fundamental matrices is because we were able to exploit the structure of the set of rank 2 matrices (Lemma~\ref{lem:bernd}). We believe that a general existence result for essential matrices would require a similar result about the variety of essential matrices.  One that still eludes us, and therefore, we are unable to say more about the existence of essential matrices for the case $5 \leq m \leq 7$. 

In theory, the non-existence of a real solution to a system of polynomials can be characterized by the {\em real Nullstellensatz} \cite{MarshallBook} and checked degree by degree via {\em semidefinite programming} \cite{boyd-vandenberghe}. Or given a $Z$ we could solve the Demazure cubics together with the linear equations cutting out $\ker_\RR(Z)$ and check if there is a real solution among the finitely many complex solutions~\cite{nister,stewenius_thesis}. In both of these approaches, it is a case by case computation for each instance of $Z$ and will not yield a characterization of those $Z$'s for which there is an essential matrix.

\subsection{Cases: $m \leq 4$}
We now consider the cases of $m \leq 4$ where $\mathcal{E}_\CC \cap \ker_\CC(Z)$ is infinite and the conventional wisdom is that an essential matrix always exists. It turns out that an essential matrix does indeed exist when $m \leq 4$. 
Again, such a result does not follow from dimension counting for complex varieties since we have to exhibit the existence of a real matrix in $\mathcal{E}_\CC \cap \ker_\CC(Z)$.

When $m \leq 3$, there is a short proof that $Z$ always has an essential matrix using the 
fact that an essential matrix can be written in the form 
$E = [t]_\times R$ where $t$ is a nonzero vector in $\RR^3$ and $R \in \textup{SO}(3)$.

\begin{theorem}
If $m\leq 3$ then $Z$ has an essential matrix.
\end{theorem}

\begin{proof}
Without loss of generality we assume $m=3$.
Choose a rotation matrix $R$ so that $y_1 \sim Rx_1$. 
Then consider $t\in \RR^3\setminus \{0\}$ which is orthogonal to both $y_2\times Rx_2$ and $y_3 \times Rx_3$. 
Now check that for each $i=1,2,3$, $y_i^\top [t]_\times R x_i = 0$ and hence 
$[t]_\times R$ is an essential matrix for $Z$. It helps to recall that $y_i^\top [t]_\times R x_i  = t^\top (y_i \times Rx_i)$. 
\end{proof}

The above argument does not extend to $m=4$.
Our main result in this section is Theorem~\ref{thm:rank4} which proves the existence of $E$ when $m = 4$ under the mild assumption that all the $x_i$'s (respectively, $y_i$'s)  are distinct. This result will need the following key lemma which is a consequence of Theorems 5.19 and 5.21 in \cite{ma-et-al}.

\begin{lemma} \label{lem:relation}
If there is a matrix $H\in \RR^{3\times 3}$ of rank at least two such that for each $i$, either $y_i \sim Hx_i$ 
or $Hx_i=0$, then $Z$ has an essential matrix.
\end{lemma}

\begin{theorem} \label{thm:rank4}
If $m=4$ and all the $x_i$'s are distinct points and all the $y_i$'s are distinct points for 
 $i=1,\ldots,4$, then $Z$ has an essential matrix.
\end{theorem}

\begin{figure*}[t!]
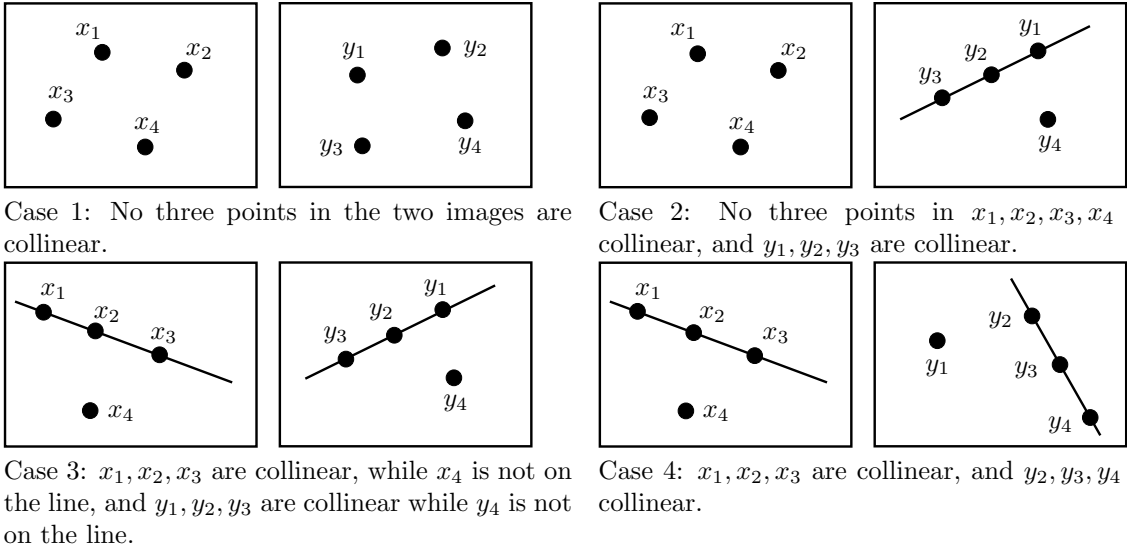

\centering
\begin{tabular}{p{0.5\textwidth}p{0.5\textwidth}}
\def\svgwidth{200pt} \input{noncollinear_pdf.tex} &
\def\svgwidth{200pt} \input{one-collinear_pdf.tex} \\
Case 1: No three points in the two images are collinear.  & 
Case 2: No three points in $x_1,x_2,x_3,x_4$ are collinear,
 and $y_1,y_2,y_3$ are collinear.\\
\def\svgwidth{200pt} \input{collinearA_pdf.tex} & 
\def\svgwidth{200pt} \input{collinearB_pdf.tex} \\
Case 3: $x_1,x_2,x_3$ are collinear, while $x_4$ is not on the line,
 and $y_1,y_2,y_3$ are collinear while $y_4$ is not on the line. & 
Case 4: $x_1,x_2,x_3$ are collinear, and $y_2,y_3,y_4$ are collinear.
\end{tabular}
\caption{The four point configurations (up to rearranging indices and swapping $x$ with $y$) for $m = 4$. Theorem~\ref{thm:rank4} proves the existence of an $E$ matrix for $m=4$ by treating each of these cases separately.}
\label{fig:m=4 E}
\end{figure*}

\begin{proof}
We divide the proof into several cases; see Figure~\ref{fig:m=4 E}. The first is the generic situation in which the $x_i$'s and $y_i$'s are in general position. In the remaining cases the input data satisfy special non-generic conditions. Together, these cases exhaust all possibilities, up to rearranging indices and swapping $x$ with $y$. 

In the first three cases, the proof proceeds by exhibiting an explicit matrix $H$ that satisfies the assumption of Lemma~\ref{lem:relation}.  Cases 1 and 2 are easy to check. 
The $H$ in case 3 is quite a bit more involved, although it only suffices to verify that it satisfies Lemma~\ref{lem:relation}, which is mechanical. The last case uses a different argument to construct an essential matrix associated to $Z$.

\begin{enumerate}

\item {\it No three of the $x_i$'s are collinear in $\RR^2$, and no three of the $y_i$'s are collinear in $\RR^2$;
see Figure~\ref{fig:m=4 E}.}

In this case, there is an invertible matrix $H\in \RR^{3\times 3}$ such that $y_i \sim Hx_i$ by Lemma~\ref{lem:non-collinear} with $n=2$. The conclusion now follows from 
Lemma~\ref{lem:relation}. \\

\item {\it No three points in $x_1,x_2,x_3,x_4$ are collinear in $\RR^2$, and  the points $y_1,y_2,y_3$ are collinear in $\RR^2$;
see Figure~\ref{fig:m=4 E}.}

By Lemma~\ref{lem:non-collinear}, we can choose an invertible matrix $H_1 \in \RR^{3\times 3}$ such that 
\begin{align*}
H_1 x_1  \sim (1,1,1)^\top, \ H_1 x_2 \sim (0,0,1)^\top, \  H_1 x_3 \sim (0,1,0)^\top \text{ and } H_1 x_4 \sim (1,0,0)^\top.
\end{align*}

On the other hand, by Lemma~\ref{lem:homography}, there is an invertible matrix $H_2 \in \RR^{3\times 3}$ such that 
\begin{align*}
H_2 y_1 =  (0,0,1)^\top, \  H_2 y_2 = (0,\alpha,1)^\top, \ H_2 y_3 = (0,\beta,1)^\top
\end{align*}
for some non-zero distinct real numbers $\alpha$ and $\beta$. Consider the rank two matrix 
\begin{align*} H_3:= \begin{pmatrix}
0 & 0 & 0 \\ 0 & -\alpha \beta & \alpha \beta \\ 0 & -\alpha &\beta
\end{pmatrix}.
\end{align*}
Then we obtain
\begin{align*}
& H_3 (1,1,1)^\top \sim H_2 y_1, \ H_3 (0,0,1)^\top \sim H_2 y_2, \\ &
 H_3 (0,1,0)^\top \sim H_2 y_3 \text{ and } H_3 (1,0,0)^\top = (0,0,0)^\top.
\end{align*}
Consequently, if we consider the rank two matrix $H:= H_2^{-1} H_3 H_1$, then $y_i \sim Hx_i$ for $i=1,2,3$ and $Hx_4 = 0$. Thus, the result follows from Lemma~\ref{lem:relation}.\\

\item
{\it The points $x_1,x_2,x_3$ are collinear in $\RR^2$ while $x_4$ is not on the line, and the points $y_1,y_2,y_3$ are collinear in $\RR^2$
while $y_4$ is not on the line; see Figure~\ref{fig:m=4 E}.}

Using Lemma~\ref{lem:homography}, 
by multiplying two invertible matrices from the left to $x_i$'s and $y_i$'s if necessary, we may assume 
$x_1 = (0,0)$, $x_2 = (0,\alpha)$, $x_3 = (0,\beta)$, $x_4 = (x_{41},x_{42})$, 
$y_1 = (0,0)$, $y_2 = (0,\gamma)$, $y_3 = (0,\delta)$ and $y_4 = (y_{41},y_{42})$, 
where $x_{41}, \alpha,\beta,\gamma,\delta$ are non-zero real numbers, $\alpha\neq \beta$ and $\gamma\neq \delta$.
Then there is a matrix $H\in \RR^{3\times 3}$ such that 
\begin{align*}
Hx_1 & = \alpha \beta x_{41} (\gamma-\delta) y_1\\
Hx_2 & = \alpha \delta x_{41}(\alpha-x_{41}) y_2\\
Hx_3 & = \beta \gamma x_{41} (\alpha-x_{41}) y_3\\
Hx_4 & = x_{41} [\beta\gamma(\alpha-x_{42}) - \alpha\delta (\beta-x_{42}) ]y_4, 
\end{align*}
given by 
$$
H := \begin{pmatrix} 
H_{11} & 0 & 0\\ H_{21} & H_{22} & 0 \\ 0 & H_{32} & H_{33}
\end{pmatrix}
$$
where
\begin{align*}
H_{11} & = (\alpha-x_{42})  \beta \gamma y_{41} - (\beta - x_{42}) \alpha \delta y_{41}\\
H_{21} & = - \alpha x_{42}\gamma\delta + \beta x_{42}\gamma\delta + \alpha\beta \gamma y_{42} - \beta x_{42}\gamma y_{42}- \alpha \beta \delta y_{42} + \alpha x_{42}\delta y_{42}\\
H_{22} & = (\alpha-\beta)x_{41}\gamma \delta\\
H_{32} & = (\alpha \delta - \beta \gamma)x_{41}\\
H_{33} & = (\gamma - \delta)x_{41}\alpha\beta.
\end{align*}
Notice that $H_{22}H_{33}\neq 0$, which implies $\rank(H)\geq 2$. Then, the result follows using Lemma~\ref{lem:relation}.\\

\item
{\it The points $x_1,x_2,x_3$ are collinear in $\RR^2$, and the points $y_2,y_3,y_4$ are collinear in $\RR^2$;
see Figure~\ref{fig:m=4 E}.}

Let $P_X$ be the plane in $\RR^3$ containing $(0,0,0)$ and the common line $l_X$ joining $x_1,x_2,x_3$.
Let $P_Y$ be the plane in $\RR^3$ containing $(0,0,0)$ and the common line joining $y_2,y_3,y_4$.
Take a unit vector $u\in P_X$ so that $u^\top x_1=0$. Let $U$ be the orthogonal matrix given by 
$$
U:= \lt( \f{x_1}{\|x_1\|}, \    u,   \ \f{x_1}{\|x_1\|}\times  u\rt)
$$
Let $w\in P_Y$ be a unit vector so that $w^\top y_4 = 0$.
We consider the orthogonal matrix 
$$W:=\lt( \f{y_4}{\|y_4\|}, \  w, \   \f{y_4}{\|y_4\|}\times w\rt).$$
Let $R$ be an orthogonal matrix so that $RU=W$, namely, $R := WU^\top$. 
Then, $R\f{x_1}{\|x_1\|}=\f{y_4}{\|y_4\|}$ and $Ru=w$. 
If $x\in l_X$, then $x = \alpha \f{x_1}{\|x_1\|}+ \beta u$ for some real numbers $\alpha,\beta$.
Thus we have
$$
Rx = \alpha R\f{x_1}{\|x_1\|}+\beta Ru
= \alpha  \f{y_4}{\|y_4\|}+\beta w\in  P_Y.
$$
Consider the essential matrix $E=[y_4]_\times R$. One has
\begin{align*}
y_4^\top E x_4 & = y_4^\top [y_4]_\times R x_4 = 0^\top R x_4= 0 \text{ and }\\
y_1^\top E x_1 & = y_1^\top [y_4]_\times R x_1 \sim y_1^\top [y_4]_\times y_4 = 0.
\end{align*}
For $i=2,3$, since $Rx_i\in P_Y = {\rm span}\{y_i,y_4\}$, one obtains
$$
y_i^\top Ex_i  \sim [y_i\times y_4]^\top Rx_i=0.
$$
Hence $E$ is an essential matrix of $Z$. 
\end{enumerate}
\end{proof}

\begin{corollary} \label{cor:exists E for m<=4}
An essential matrix always exists when $m \leq 4$ provided all the $x_i$'s are distinct and all the $y_i$'s are distinct. 
\end{corollary}

\section{Discussion}
\label{sec:discussion}
In this paper, we have settled the existence problem for fundamental matrices for all values of $m$ and essential matrices for $m \le 4$ (equivalently, for all $m$ for which $\rank(Z) \leq 4$).
In doing so, we have shown that pure dimension counting arguments are not enough to reason about the existence of real valued epipolar matrices. 

As we mentioned in the previous section, the conditions for the existence of an essential matrix for $5 \le m \le 7$ appear to be difficult, and are unknown for $m=5,6$. 
For $m=7$, we did find a test for the existence of an essential matrix. This uses the classical theory of {\em Chow forms}~\cite{dalbec-sturmfels,gelfand-et-al}. We have not included it in this paper since we felt that it deserves further attention and can possibly be simplified.
The interested reader can find the details in \cite[Section 3.3]{epipolar_paper_v1}. Chow forms also provide a test for 
whether $\ker_\CC(Z) \cap \mathcal{E}_\CC \neq \emptyset$ when $m=6$ \cite[Section 3.1]{epipolar_paper_v1}. Again, we have left this out of the current paper since it does not answer the question of existence of a real essential matrix when $m=6$.

Even though our results are phrased in terms of the matrix $Z$, we have shown that they can be reinterpreted in terms of the input $X$ and $Y$ in most cases. We are curious about the set of six and seven point correspondences in two views for which there is no fundamental matrix. Theorems~\ref{thm:m=7 and kernel in rank one} and \ref{thm:m=6 and kernal in rank one} characterized the point configurations for which there is no fundamental matrix because $\ker_\RR(Z) \subseteq \mathcal{R}_1$. It would also be interesting to understand the configurations for which $\ker_\RR(Z)$ contains only matrices of ranks one and three 
as in Example~\ref{example:rank one and three}.

The results in this paper show that reasoning over real numbers is both a source of surprises and complications. We believe that similar surprises and complications lurk in other existence problems in multiview geometry and are worthy of study.

\appendix

\section{Proof of Lemma~\ref{lem:homography}} \label{app:homography}
Since $l-x_0$ and $m-y_0$ are lines in $\RR^2$ passing through the origin, one can choose an orthogonal matrix $W\in \RR^{2\times 2}$ such that 
$m-y_0 = W(l-x_0)$. It follows that
$$
m= W(l-x_0)+y_0 = Wl - Wx_0+y_0 = Wl +z
$$
where $z:= y_0-Wx_0$ is a point in $\RR^2$.
Then, for the $3\times 3$ matrix 
$H:= \lt(\begin{smallmatrix} W & z \\ 0 & 1\end{smallmatrix}\rt)$, one has 
$\lt(\begin{smallmatrix} m \\ 1 \end{smallmatrix}\rt) = H \lt(\begin{smallmatrix} l \\ 1 \end{smallmatrix}\rt)$ which verifies the statement \eqref{lem:homography2}. 
In addition, $H \lt(\begin{smallmatrix} x_0 \\ 1 \end{smallmatrix}\rt) = \lt(\begin{smallmatrix} y_0 \\ 1 \end{smallmatrix}\rt)$, and thus the assertion \eqref{lem:homography1} also holds. \qed

%
%
%
%
%
\section{Proof of Lemma~\ref{lemma:chum}}
\label{sec:chum}
Recall that a fundamental matrix can be written in the form
$
F = [b]_\times H
$
where $b$ is a nonzero vector in $\RR^3$ and $H\in \RR^{3\times 3}$ is an invertible matrix. Then the epipolar constraints can be rewritten as
\begin{align}
& y_i^\top F x_i = 0, \forall i = 1, \hdots, m. \nonumber\\
\iff &y_i^\top [b]_\times Hx_i = 0, \forall i = 1, \hdots, m. \nonumber\\
\iff &y_i^\top(b \times H x_i) = 0, \forall i = 1, \hdots, m. \label{eq:before}\\
\iff &b^\top(y_i \times H x_i) = 0, \forall i = 1, \hdots, m.\label{eq:after}\\
\iff &b^\top\begin{pmatrix} \cdots & y_i \times H x_i & \cdots \end{pmatrix} = 0. \nonumber
\end{align}
A non-zero $b$ exists in the expression for $F$  if and only if 
\begin{align}\label{eq:rank}
\rank \begin{pmatrix} \cdots & y_i \times H x_i & \cdots \end{pmatrix} < 3.
\end{align}
The equivalence of~\eqref{eq:before} and~\eqref{eq:after} follows from the fact that $p^\top (q \times r) = -q^\top(p \times r)$. The matrix in~\eqref{eq:rank} is of size $3 \times m$. A sufficient condition for it to have rank less than 3 is for $m-2$ or more columns to be equal to zero. This is the case if we take $H=A$ given in the assumption. \qed

\vspace{0.2cm}
The observation about the scalar triple product and the resulting rank constraint 
has also been used by Kneip {\em et al.}~\cite{kneip-et-al} but only in the calibrated case.

\section{Proof of Theorem~\ref{thm:m=7 and kernel in rank one}}
\label{sec:m=7 and kernel in rank one}
{\it ``If" part}: Suppose \eqref{thm:m=7 and kernel in rank one, part 1} holds and let $\tau$ be the set given in 
\eqref{thm:m=7 and kernel in rank one, part 1}.
Then there is a $u\in \PP_\RR^2$ such that $u^\top y_i=0$ for any $i\in \tau$. Let $x_k$ be the single element in the set 
$\{x_i\}_{i\notin \tau}$. Consider a basis $\{v_1,v_2\}\subseteq \PP_\RR^2$ of the orthogonal complement of $x_k$.
For $j=1,2$, define $A_j = uv_j^\top\in \PP_\RR^{3\times 3}$ and let $a_j\in \PP_\RR^8$ be its vectorization. Then $\{a_1,a_2\}$ is a linearly independent set spanning
a subset of $\mathcal{R}_1$. Moreover for any $i=1,\ldots,7$ and $j=1,2$, $y_i^\top A_j x_i = (y_i^\top u)(v_j^\top x_i)=0$. 
Hence $a_j\in \ker_\RR(Z)$ for $j=1,2$.
As $\rank(Z)=7$ (cf.~\eqref{eqn:rank Z is m}), ${\rm ker}_\RR(Z) ={\rm span}\{a_1,a_2\}\subseteq \mathcal{R}_1$.
The same idea of proof works if \eqref{thm:m=7 and kernel in rank one, part 2} holds.

{\it ``Only if" part}: Consider a basis $\{a_1,a_2\}\subseteq \PP_\RR^8$ of ${\rm ker}_\RR(Z)$, which is inside $\mathcal{R}_1$, and 
assume $a_j$ is the vectorization of $A_j\in \PP_\RR^{3\times 3}$ for $j=1,2$. For any $j$, $\rank(A_j)=1$, 
so $A_j = u_jv_j^\top$ for some $u_j,v_j\in \PP_\RR^2$. Since $\rank(A_1+A_2)=1$, a simple check shows that either $\{u_1,u_2\}$ or 
$\{v_1,v_2\}$ is linearly dependent. Thus, up to scaling, we may assume either $u_1=u_2$ or $v_1=v_2$. 
If $u_1=u_2$, then $\{v_1,v_2\}$ is linearly independent. In addition, 
$0=y_i^\top A_j x_i = (y_i^\top u)(v_j^\top x_i)$ for each $i=1,\ldots,6$, $j=1,2$. Thus, either $y_i^\top u=0$ or 
$x_i\in {\rm span}\{v_1,v_2\}^\perp$. Notice that ${\rm span}\{v_1,v_2\}^\perp$ is a singleton in $\PP_\RR^2$.
As $\rank(Z)=7$, by the paragraph after Lemma~\ref{lem:rank bound}, neither `` $y_i^\top u=0$ for all $i$" nor
``$x_i\in {\rm span}\{v_1,v_2\}^\perp$ for all $i$" can happen. Hence \eqref{thm:m=7 and kernel in rank one, part 1}
holds with the nonempty proper subset $\tau:= \{i \ : \  y_i^\top u = 0\}$ of $\{1,\ldots,7\}$. If $v_1=v_2$, by the same idea one sees that 
\eqref{thm:m=7 and kernel in rank one, part 2} holds. \qed

\section{Proof of Theorem~\ref{thm:m=6 and kernal in rank one}}
\label{sec:m=6 and kernal in rank one}

Recall that we are assuming that $Z$ has full row rank, i.e., $m=\rank(Z)=6$. By Lemma \ref{lem:rank bound}, this can only be true for $m=6$ if $x_i$ and $y_i$ are not simultaneously collinear, i.e. one of $X$ or $Y$ has to have full row rank.

{\em ``If" part}: If all points $y_i$ are collinear in $\RR^2$, then there is $u\in \PP_\RR^2$ such that $u^\top y_i=0$ for any $i=1,\ldots,6$. 
Let $e_1=(1,0,0)^\top$, $e_2 = (0,1,0)^\top$, $e_3 = (0,0,1)^\top$. Consider the $3\times 3$ matrices
$$
A_j = ue_j^\top \text{ for } j = 1,2,3
$$
and their vectorizations $a_j \in \PP_\RR^8$. Then, $\{a_1,a_2,a_3\}$ is a linearly independent set spanning a subset of $\mathcal{R}_1$.
Moreover, for any $i=1,\ldots,6$ and $j=1,2,3$, $y_i^\top A_j x_i = (y_i^\top u)(x^\top_i e_j)=0$. 
Hence $a_j\in \ker_\RR(Z)$. As $\rank(Z)=6$ (cf.~\eqref{eqn:rank Z is m}), ${\rm ker}_\RR(Z) = {\rm span}\{a_1,a_2,a_3\}\subseteq \mathcal{R}_1$. 
The same idea of proof works if all points $x_i$ are collinear in $\RR^2$.

{\em ``Only if" part}: Consider a basis $\{a_1,a_2,a_3\}\subseteq \PP_\RR^8$ of ${\rm ker}_\RR(Z)$, which is inside $\mathcal{R}_1$, and 
assume $a_j$ is the vectorization of $A_j\in \PP_\RR^{3\times 3}$ for $j=1,2,3$. Then, by Lemma~\ref{lem:meshulam} 
with $n=3$ and $r=1$, up to taking transpose of all $A_j$, there are nonzero vectors $u,v_1,v_2,v_3\in \PP_\RR^2$ such that 
$A_j = uv_j^\top$ for $j=1,2,3$. The vectors $v_j$ are linearly independent as $A_j$ are. 
Moreover 
$0=y_i^\top A_j x_i = (y_i^\top u)(x_i^\top v_j)$ for any $i=1,\ldots,6$, $j=1,2,3$. 
We fix $i\in \{1,\ldots,6\}$ and claim that $y_i^\top u=0$. Indeed, if $y_i^\top u \neq 0$, then $x_i^\top v_j= 0$ for any $j=1,2,3$. 
As vectors $v_j$ are linearly independent we have $x_i=0$. This is impossible because $x_i$ as a point in $\PP_\RR^2$ has nonzero third coordinate. Hence our claim is true and thus all points $y_i$ are collinear in $\RR^2$. 
If it is necessary to replace $A_j$ by $A_j^\top$, it follows that all points $x_i$ are collinear in $\RR^2$.   \qed

\section{Proof of Lemma~\ref{lem:bernd}} \label{app:lemma bernd}
We first consider the case when $L$ is a projective line, i.e., $$L = \{A\mu + B\eta \,:\, \mu,\eta \in \RR \}$$  for some  $A,B \in \RR^{3 \times 3}$, with $B$ invertible. Then $B^{-1}L = \{ B^{-1}A\mu + I\eta \,:\,\mu,\eta \in \RR \}$ is an isomorphic image of $L$ and contains a matrix of rank two if and only if $L$ does.
 Hence we can assume $L = \{ M \mu - I \eta \,:\, \mu,\eta \in \RR\}$ for some $M \in \RR^{3 \times 3}$. The homogeneous cubic polynomial $\det(M \mu - I\eta)$ is not identically zero on $L$.
 When dehomogenized by setting $\mu=1$, it is the characteristic polynomial of $M$. Hence the three roots of $\det(M \mu - I \eta)=0$ in $\PP^1$ are $(\mu_1,\eta_1) \sim (1,\lambda_1), (\mu_2,\eta_2) \sim (1,\lambda_2)$ and $(\mu_3,\eta_3) \sim (1,\lambda_3)$ where $\lambda_1, \lambda_2, \lambda_3$ are the eigenvalues of $M$. At least one of these roots is real since $\det(M\mu - I\eta)$ is a cubic. Suppose $(\mu_1,\eta_1)$ is real. If $\rank(M\mu_1-I\eta_1) = \rank(M-I\lambda_1) = 2$, then $L$ contains a rank two matrix.  Otherwise, $\rank(M-I\lambda_1) = 1$. Then $\lambda_1$ is a double eigenvalue of $M$ and hence equals one of $\lambda_2$ or $\lambda_3$. Suppose $\lambda_1 = \lambda_2$. This implies that $(\mu_3,\eta_3)$ is a real root as well. If it is different from $(\mu_1,\eta_1)$, then it is a simple real root. Hence, $\rank(M\mu_3-I\eta_3)=2$, and $L$ has a rank two matrix.
So suppose $(\mu_1,\eta_1) \sim (\mu_2, \eta_2) \sim (\mu_3,\eta_3) \sim (1,\lambda)$ where $\lambda$ is the unique eigenvalue of $M$.
In that case, $\det(M\mu - I\eta) = \alpha \cdot (\eta-\lambda \mu)^3$ for some constant $\alpha$.
This finishes the case $\dim(L)=1$.

Now suppose $\dim(L) \geq 2$. If $\det$ restricted to $L$ is not a power of a homogeneous linear polynomial, then there exists a projective line $L'$ in $L$ such that $\det$ restricted to $L'$ is also not the power of a homogeneous linear polynomial.
The projective line $L'$ contains a matrix of rank two by the above argument. \qed

\section{A proof for the existence of a fundamental matrix when $m \leq 4$} \label{appendix:fundamental}
\begin{theorem}
If $m\leq 4$, then $Z$ has a fundamental matrix.
\end{theorem}
\begin{proof}
If $m\leq 3$, by adding point pairs if necessary we can assume $m=3$.
One can always construct an invertible matrix $H$ such that $y_1 \sim H x_1$ which implies that $y_1 \times H x_1 = 0$ and equation~\eqref{eq:rank} is satisfied.

Let us now consider the case $m = 4$.  Since $\rank(Z) = 4$, by Lemma~\ref{lem:rank bound}, $\rank(X) \ge 2$ and $\rank(Y) \ge 2$. 
If we can find two indices $i$ and $j$ such that the matrices $\begin{pmatrix} x_i & x_j\end{pmatrix}$ and $\begin{pmatrix} y_i & y_j\end{pmatrix}$ both have rank 2 then we can construct an invertible matrix $H$ such that $y_i \sim H x_i$ and $y_j \sim H x_j$ and that would be enough for~\eqref{eq:rank}.
Without loss of generality let us assume that the matrix $\begin{pmatrix} x_1 & x_2\end{pmatrix}$ is of rank 2, i.e., $x_1 \not\sim x_2$. If $\begin{pmatrix} y_1 & y_2\end{pmatrix}$ has rank 2 we are done. So let us assume that this is not the case and $y_2 \sim y_1$. Since $\rank(Y) \ge 2$, we can without loss of generality assume that $y_3 \not\sim y_1$. Since $x_1 \not\sim x_2$, either, $x_3 \not\sim x_1$ or $x_3 \not\sim x_2$.
In the former case, $i = 1, j = 3$ is the pair we want, otherwise $i = 2, j = 3$ is the pair we want. 
\end{proof}

\bibliographystyle{plain}
\bibliography{agpc}

\end{document}